\documentclass[conference]{llncs}

\usepackage{cite}
\usepackage{amsmath,amssymb,amsfonts}
\usepackage{graphicx}
\usepackage{textcomp}
\usepackage{xcolor}
\usepackage{algorithm}
\usepackage{algpseudocode}
\usepackage{amsmath}
\usepackage{makecell}
\usepackage{booktabs}
\usepackage{multirow}
\usepackage{wrapfig}
\usepackage{hyperref}
\hypersetup{
    colorlinks=true,
    linkcolor=blue
}

\newcommand\blfootnote[1]{%
  \begingroup
  \renewcommand\thefootnote{}\footnote{#1}%
  \addtocounter{footnote}{-1}%
  \endgroup
}

\newcommand{\tablelinespace}{\addlinespace[1mm]}

\def\BibTeX{{\rm B\kern-.05em{\sc i\kern-.025em b}\kern-.08em
    T\kern-.1667em\lower.7ex\hbox{E}\kern-.125emX}}
\begin{document}

\title{LipBaB: Computing exact Lipschitz constant of Relu networks
}

\author{Aritra Bhowmick
\and
Meenakshi D'Souza
\and
G. Srinivasa Raghavan
}
\authorrunning{Bhowmick et al.}
%
\institute{International Institute of Information Technology Bangalore\\
Bangalore, India}

\maketitle

\blfootnote{Research sponsored by DRDO Headquarters, New Delhi, India.}
\blfootnote{Corresponding author email: meenakshi@iiitb.ac.in}

\begin{abstract}
The Lipschitz constant of neural networks plays an important role in several contexts of deep learning ranging from robustness certification and regularization to stability analysis of systems with neural network controllers. Obtaining tight bounds of the Lipschitz constant is therefore important.  We introduce LipBaB, a branch and bound framework to compute certified bounds of the local Lipschitz constant of deep neural networks with ReLU activation functions up to any desired precision. It is based on iteratively upper-bounding the norm of the Jacobians, corresponding to different activation patterns of the network caused within the input domain. Our algorithm can provide provably exact computation of the Lipschitz constant for any $p$-norm.

\end{abstract}

\section{Introduction}

The notion of Lipschitz constant for a function, in general, bounds the rate of change of outputs with respect to the inputs. For neural networks, the Lipschitz constant of the network is an useful metric to measure sensitivity, robustness and many other properties. It has several applications in the context of deep learning \cite{deepmind}. It can be used as a regularization constraint while training or provide certified robustness bounds against adversarial perturbations \cite{lipschitzmargin}. It also helps in providing guaranteed generalization bounds \cite{generalizationBounds}. Other use cases involve estimating Wasserstein distance \cite{estWassDist}, stabilising training of GANs \cite{spectralNorm}, and acting as building blocks for formulating invertible neural networks and flow based generative models \cite{invres}, \cite{flowres}. Hence, a provably correct and accurate Lipschitz constant estimation technique is important.  

There have been several recent works related to the estimation of Lipschitz constants \cite{lipmilp},\cite{fastlip}, \cite{NIPS2018-seqlip}, etc. Section~\ref{sec:related-work}
provides a brief study on these works, including a comparison with our algorithm. 

It has been shown that exactly computing the Lipschitz constant or even within any desired approximation is computationally hard \cite{NIPS2018-seqlip}, \cite{lipmilp}. Hence computing the exact Lipschitz constant for large networks is almost infeasible. However, any technique which provides iteratively refined bounds with convergence over time is desired. Techniques used for the verification of neural networks including robustness certification, output bounds estimation and calculation of the Lipschitz constant are deeply interlinked and share many similarities among them. Our work also uses some of the extensions and combinations of related techniques like symbolic propagation, interval arithmetic and linear programming (for feasibility checking), within our proposed branch and bound framework.

Our Branch and Bound (BaB) algorithm is based on iterative space partitioning and upper bounding the local Lipschitz constant for each such partition. Each node in the BaB tree has associated with it, an input partition defined by some half-space constraints and a activation pattern of the network on which the Lipschitz upper bounds are calculated. At each iteration our algorithm provides a refined upper bound of the exact Lipschitz constant until convergence. Preliminaries and notations used are given in Section~\ref{sec:preliminaries}. Section~\ref{sec:overview} gives an overview of the overall approach. Sections~\ref{sec:algorithm} and \ref{sec:final-algorithm} provides details of our algorithm. Section~\ref{sec:experiments} provides the implementation and experimental results demonstrating the performance on different parameters.

\section{Related Work}
\label{sec:related-work}
We provide a brief comparison of some of the related works and the specific settings they apply to in the table below \footnote{We have used the names of the techniques as given in the respective papers. The column titled \textbf{LipBaB} is our algorithm}. To the best of our knowledge LipBaB is the first work which is able to calculate the \textbf{exact} local Lipschitz constant for \textbf{any} $p$-norm. Our algorithm can also be used to compute the \textbf{global} Lipschitz constant.
    
\begin{table}[]
\begin{tabular}{c|ccccccc}
\textbf{}    & \textbf{LipBaB}           & \textbf{LipMIP}            & \textbf{LipSDP}                & \textbf{SeqLip}                & \textbf{CLEVER}                & \textbf{LipOpt}            & \textbf{FastLip} \\ \hline
global/local & \multicolumn{1}{c|}{local, global} & \multicolumn{1}{c|}{local} & \multicolumn{1}{c|}{global}    & \multicolumn{1}{c|}{global}    & \multicolumn{1}{c|}{local}     & \multicolumn{1}{c|}{local} & local            \\
gurantee     & \multicolumn{1}{c|}{exact} & \multicolumn{1}{c|}{exact} & \multicolumn{1}{c|}{upper}     & \multicolumn{1}{c|}{heuristic} & \multicolumn{1}{c|}{heuristic} & \multicolumn{1}{c|}{upper} & upper            \\
p-norms      & \multicolumn{1}{c|}{p}     & \multicolumn{1}{c|}{1,inf} & \multicolumn{1}{c|}{2}         & \multicolumn{1}{c|}{p}         & \multicolumn{1}{c|}{p}         & \multicolumn{1}{c|}{p}     & p                \\
activations  & \multicolumn{1}{c|}{ReLu}  & \multicolumn{1}{c|}{ReLu}  & \multicolumn{1}{c|}{ReLu,Diff} & \multicolumn{1}{c|}{ReLu}      & \multicolumn{1}{c|}{ReLu,Diff} & \multicolumn{1}{c|}{Diff}  & ReLu            
\end{tabular}
\end{table}

As shown in the above table, each of the approaches on estimating Lipschitz constants meet different requirements, depending on the kind of norms that can be evaluated, or the global or local nature of the Lipschitz constant or whether they are upper bounds or lower bounds or heuristic estimations. Difference also arises on the kind of activation functions that are supported. 
Fastlip provides an efficient way of upper-bounding the local Lipschitz constant of ReLU-networks using interval-bound propagation \cite{fastlip}. LipSDP uses semi-definite programming to provide upper bounds of $l_2$-Lipschitz constant \cite{lipsdp}. CLEVER is a heuristic approach which uses extreme value theory and sampling techniques to find a approximation of the Lipschitz constant \cite{clever}. SeqLip transforms this problem into a combinatorial optimization problem and uses greedy strategies to provide estimates \cite{NIPS2018-seqlip}. LipOpt uses polynomial optimization to find global or local upper bounds of continuously differentiable networks \cite{lipopt}. LipMip is able to calculate the exact Lipschitz constant under $l_0, l_{\infty}$ norms by formulating the problem as a Mixed Integer problem \cite{lipmilp}. LipMip can be extended to other linear norms also.

\section{Preliminaries}
\label{sec:preliminaries}
This section provides preliminaries, notations and background definitions as required by the problem. 

\subsection{Vectors, Matrices and Intervals}
\begin{table}[h!]
    \centering
    \begin{tabular}{cl}
      $\underline{a},\overline{a}$  &  \makecell[tl]{lower and upper bounds of $a$}\\
      \tablelinespace
      $[a],[a]_i$ & \makecell[tl]{interval vector and its $i^{th}$ component $[\underline{a}_i,\overline{a}_i]$}\\
      \tablelinespace
      $[A],[A]_{ij}$ & \makecell[tl]{interval matrix and its $(i,j)^{th}$ element $[\underline{A}_{ij},\overline{A}_{ij}]$}\\
    \end{tabular}
\end{table}
Addition and multiplication operations on intervals, where $p,q$ are intervals and $c$ is a scalar
\begin{align*}
    cp &\equiv [min(c\underline{p},c\overline{p}),max(c\underline{p},c\overline{p})]\\
    c+p &\equiv [c+\underline{p},c+\overline{p}]\\
    p+q &\equiv [\underline{p}+\underline{q},\overline{p}+\overline{q}]\\
    pq &\equiv [min(\underline{p}\underline{q},\underline{p}\overline{q},\overline{p}\underline{q},\overline{p}\overline{q}), max(\underline{p}\underline{q},\underline{p}\overline{q},\overline{p}\underline{q},\overline{p}\overline{q})]\\
\end{align*}

\subsection{Operator norms}

The $p$-norm, $\|\cdot\|_p$, of a vector $x$ is defined as: \[ \|x\|_p=\left(\sum_{i=1}^n |x_i|^p \right)^{\frac{1}{p}}\]
Let $A$ be a matrix of size $m\times n$. Then the operator norm of this matrix, induced by the vector norm $\|\cdot\|_p$, is defined as:
\[ \|A\|_p =\sup_{x \in \mathbb{R}^n}\|Ax\|_p, \|x\|_p=1  \]
The following lists the values of $p$-norm for some commonly used values of $p$. 

\[ \|A\|_p=
\begin{cases}
max_{1\leq j\leq n}\sum_{i=1}^m |A_{ij}|, & \text{if } p=1\\
\sigma_{max} & \text{if } p=2\\
max_{1\leq i\leq n}\sum_{j=1}^n |A_{ij}|, & \text{if } p=\infty
\end{cases}
\]
where $\sigma_{max}$ is the largest singular value of the matrix $A$.

Note that $\|A\|_2 \leq \|A\|_F=\left( \sum_{i=1}^m \sum_{j=1}^n |A_{ij}|^2\right)^{1/2}$, where the $\|A\|_F$ is the Frobenius norm of $A$. 

\subsection{Generalized Jacobian}
The Jacobian  of a function $f:\mathbb{R}^n \rightarrow \mathbb{R}^m $ at a differentiable point $x$ is given as:
\[ J_f(x)= 
\begin{pmatrix}
\frac{\delta f_1}{\delta x_1} & \frac{\delta f_1}{\delta x_2} &  \dots & \frac{\delta f_1}{\delta x_n}\\
\frac{\delta f_2}{\delta x_1} & \frac{\delta f_2}{\delta x_2} & \dots & \frac{\delta f_2}{\delta x_n}\\
. & . & . & .\\
\frac{\delta f_m}{\delta x_1} & \frac{\delta f_m}{\delta x_2} & \dots & \frac{\delta f_m}{\delta x_n}
\end{pmatrix}
\]
For functions which are not continuously differentiable, we have the notion of Clarke's generalized Jacobian. The generalized Jacobian of such a function $f$ at a point $x$ is defined as: 
\[ \delta_f(x)=\textit{co}\{\lim_{x_i\rightarrow x} J_f(x_i): x_i \textit{ is differentiable}\}\]
In other words, $\delta_f(x)$ is the convex hull of the set of Jacobians of nearby differential points. For a differentiable point $\delta_f(x)$ is a singleton set $\{J_f(x)\}$. 

\subsection{Lipschitz Constant and norms of Jacobians}

For a locally Lipschitz continuous function $f: \mathbb{R}^n \rightarrow \mathbb{R}^m$ defined over an open domain $\mathcal{X} \in \mathbb{R}^n$,
the local Lipschitz constant $\mathcal{L}_{p}(f,\mathcal{X})$ is defined as the smallest value such that 
\[ \forall x,y\in \mathcal{X}: \|f(y)-f(x)\|_p \leq \mathcal{L}_{p}(f,\mathcal{X}) \cdot \| y-x\|_p \]

For a differentiable and locally Lipschitz continuous function, the Lipschitz constant is given as \cite{deepmind}: 
\[\mathcal{L}_p(f,\mathcal{X})| = sup_{x\in \mathcal{X}}\|J_f(x)\|_p , \text{(Federer, 1969)} \] 
where, $\|J_f(x)\|_p$ is the induced operator norm on the matrix $J_f(x)$.

However, in the case of ReLU-networks the function is piece-wise linear in nature and is not differentiable everywhere. For such functions, the above definition of Lipschitz constant can be extended accordingly with the help of Clarke's generalized Jacobian \cite{lipmilp}:
\[\mathcal{L}_p(f,\mathcal{X}) = \sup_{M\in \delta_f(x), x \in \mathcal{X}}\|M\|_p = \sup_{x_d \in \mathcal{X}}\|J_f(x_d)\|_p\]
where, $x_d$ is a differentiable point in $\mathcal{X}$.

It is natural that the $ \sup_{M\in \delta_f(x), x \in \mathcal{X}}\|M\|_p$ is attained at a differentiable point $x_d$ in $\mathcal{X}$. By definition, a norm $\|.\|$ is convex. Therefore the maximal value of the norm of the elements in $\delta_f(x)$, which is itself a convex set, is attained at one of the extreme points which are the Jacobians of nearby differentiable points.

This result shows that for computing the Lipschitz constant we don't need to necessarily account for the non-differentiable points.

\subsection{Upper bounds on Jacobian norms}
\begin{lemma}
If $A$ and $B$ are both matrices of size $m\times n$, and if for each $i,j$, $|A_{ij}| \leq B_{ij}$, then $\|A\|_p \leq \|B\|_p$.
\end{lemma}
\begin{proof}
Let $x^*= \arg\max_{x\in \mathbb{R}^n}\|Ax\|_p$ where $\|x\|_p=1$, and $x'$ be a vector such that $x'_i=|x^*_i|$ . Also we note from the definition of $p$-norm of a vector, that for any two vectors $u,v$ if $|u_i| \leq v_i$ for each $i$, then $\|u\|_p \leq \|v\|_p$. The following chain of inequalities establishes the lemma:
\begin{align*}
|(Ax^*)_i| &= |\sum_{j=1}^n A_{ij}.x^*_j|
\leq \sum_{j=1}^n|A_{ij}.x^*_j| =\sum_{j=1}^n|A_{ij}| \cdot |x^*_j|
\\&\leq \sum_{j=1}^n B_{ij}.|x^*_j| = \sum_{j=1}^n B_{ij}x'_j 
= (Bx')_i
\end{align*}
for each $i$, where $1 \leq i \leq m$.
Therefore, 
\[ \|A\|_p= \|Ax^*\|_p \leq \|Bx'\|_p \leq sup_x\|Bx\|_p=\|B\|_p   \]
\end{proof}

The above result can be used to upper bound the Lipschitz constant by upper bounding the absolute values of the partial derivatives \cite{fastlip}.

\begin{lemma}\label{thm:ub}
Let $U$ be a matrix of the same size as the Jacobian $J_f(x)$ of a function $f(x)$. If $U$ is such that $sup_{x_d\in\mathcal{X}}|J_f(x_d)_{ij}| \leq U_{ij}$ for all $i,j$, then $\mathcal{L}_p(f,\mathcal{X})\leq \|U\|_p$, in the open domain $\mathcal{X}$.
\end{lemma}
\begin{proof}
We know that $\mathcal{L}_p(f,\mathcal{X})=sup_{x_d\in\mathcal{X}}\|J_f(x_d)\|_p$.
Now,
\[ \mathcal{L}_p(f,\mathcal{X})=\sup_{x_d\in\mathcal{X}}\|J_f(x_d^*)\|_p \leq \|[sup_{x_d\in\mathcal{X}}|J_f(x_d)_{ij}|]\|_p \leq \|U\|_p  \]
where $x_d^*$ is $\arg\max_{x_d\in\mathcal{X}}\|J_f(x_d)\|_p$.
\end{proof}

\subsection{Feed forward ReLU networks and their Jacobians}

Deep feed-forward ReLU networks (MLPs) are stacked layers of perceptrons with ReLU activation functions. The following notations are used for the networks.

\begin{table}[h!]
    \centering
    \begin{tabular}{cl}
      $L$ &  number of layers in the network\\
      \tablelinespace
      $n^{(l)}$ &  number of neurons in  the $l^{th}$ layer\\
      \tablelinespace
      $x^{(l)}$ & vector of outputs from the neurons at layer $l$\\
      \tablelinespace
      $z^{(l)}$ & vector of inputs to the neurons at layer $l$\\
      \tablelinespace
      $W^{(l)}, b^{(l)}$ & Weights and biases of the $l^{th}$ layer of a network\\
    \end{tabular}
\end{table}

Given an input vector $x$ and a list of parameters $\theta\equiv\{W^{(l)}, b^{(l)}, i=1,\ldots,n\}$ describing the network architecture, the function $f:\mathbb{R}^n \rightarrow \mathbb{R}^m$ represented by the ReLU network is defined as follows. 
\begin{multline*}
f(x,\theta) = W^{(L)}\phi(W^{(L-1)}(\dots\phi(W^{(1)}x+b^{(1)})\dots)+ b^{(L-1)})+b^{(L)}
\end{multline*}
where $\phi$ denotes the ReLU activation function with $\phi(x)=max(0,x)$. ReLU networks are piece-wise linear in nature.
The concept of Jacobians for a network (with respect to the outputs and inputs of $f$) gives us an idea about how the network outputs vary with changes in inputs near a point. The Jacobian at a point $x$ is calculated by the chain rule of derivatives and is done using back propagation. It is important to note that this Jacobian is defined only if the derivatives at every ReLU node is defined. This happens only if the input to each ReLU node is strictly positive or negative. If it is equal to zero then a sub-gradient exists which lies between $[0,1]$. The Jacobian at a point $x$, if defined, can be compactly represented as:
\[J_f(x) = W^{(L)} \Lambda^{(L-1)}W^{(L-1)}\ldots \Lambda^{(1)}W^{(1)}\]
where $\Lambda^l$ encodes the activation pattern of a layer $l$ caused by the input $x$. It is a diagonal matrix, having $1$s as elements if the  corresponding neuron is active, otherwise $0$s for inactive neurons. The Jacobian is the same for all the points strictly inside a linear region with the same activation pattern. Since ReLU networks are piece-wise linear in nature, the Lipschitz constant is exactly equal to the $p$-norm of the Jacobian at one such linear region in the input domain.

\section{Approach overview}
\label{sec:overview}
The proposed algorithm is composed of several components like initial estimation of activation pattern, calculation of Lipschitz bounds and partitioning of sub-problems, which are unified within the branch and bound framework.
This section provides an overview of the algorithm and some of its components. The next section describes each of these components in detail, including the representation of the sub-problems, in a relevant order.

\begin{algorithm}
    \caption{Overview}
    \label{alg:branchandbound}
    \begin{algorithmic}[1]
        \Procedure{Algo}{network $N$, input-domain $\mathcal{X}$}
            \State Set initial half-space constraints and initial activation pattern considering the input domain $\mathcal{X}$
            \State Create new sub-problem and get initial upper-bound  estimation of $\mathcal{L}_p(f,\mathcal{X})$
            \State Initialize the set of sub-problems with the sub-problem
            \While{\textbf{not} Terminate}
                \State Select sub-problem to branch
                \State Branch into new sub-problems and get refined upper-bound estimations
                \State Remove the selected sub-problem and add the new sub-problems in the problem set
            \EndWhile
            \State\Return {Final estimate of the Lipschitz constant}
        \EndProcedure
    \end{algorithmic}
\end{algorithm}

A short overview of some of the components is given below.
\begin{itemize}
    \item \texttt{Symbolic Propagation}: Uses symbolic expressions and intervals to compute the initial estimation of the activation pattern. It also provides us with output bounds of the network for the given input region. We also use symbolic propagation in a slightly different way to generate half-space constraints.
    \item \texttt{Lipschitz Bounds}: Computes an upper bound to $\mathcal{L}_p(f,\psi)$ for a partition $\psi \in \mathcal{X}$ using interval arithmetic.
    \item \texttt{Branching}: Branches a given sub-problem $\rho$ into new sub-problems and computes their Lipschitz upper-bound. Each sub-problem has a partition of the input space associated with it.
\end{itemize} 

\section{Components of the algorithm}
\label{sec:algorithm}
\subsection{Input Domain Abstraction and Interval bound propagation}

We consider the input region $\mathcal{X}$ as a hyper-rectangle of $n$-dimensions. It can be represented as the Cartesian product of the following intervals.
\[\mathcal{X}=[\underline{x_1},\overline{x_1}]\times [\underline{x_2},\overline{x_2}]\times\dots\times [\underline{x_n},\overline{x_n}]\]
where $x_i$ denotes the $i$th dimension of an input $x$.

The main purpose of this section is to get an initial estimation of the activation states of neurons with respect to all inputs in the input region. We do this by calculating the pre-activation bounds of each neuron. Additionally we obtain the output bounds of the network considering all the input regions as well. We mark the state of a neuron as active/inactive based on whether the pre-activation bounds are positive/negative. If the pre-activation bounds contain both negative and positive values, it is marked as $\ast$ or an undecided neuron. These neurons may be active/inactive depending on the inputs.

The pre-activation bounds of each neuron can be calculated using interval bound propagation. For all $l=1,\ldots,n$, taking $[x^{(0)}]=[z^{(0)}]$

\[[z^{(l)}] = W^{(l)}[x^{(l-1)}] + b^{(l)}\]
\[[x^{(l)}_i] = [\max(0,\underline{z^{(l)}_i}),\max(0,\overline{{z}^{(l)}_i} )]\]

where $[z^{(0)}]$ denotes the input interval vector.

However these bounds are an over-approximation of the actual range of values and these over-approximations accumulate across the layers because of the dependency problem in interval arithmetic. The dependency problem occurs because multiple occurrences of the same variables are treated independently.  This dependency problem can be reduced by using symbolic expressions. The symbolic expression makes use of the fact that the input to any neuron depends on the same set of input variables, instead of considering only the outputs from the previous layer independently as done in naive interval bound propagation. A similar approach is presented in \cite{reluval}, but, a crucial difference is that we maintain a single expression instead of two, per neuron. Also we create new symbols instead of concretizing the bounds completely for $\ast$-neurons, to reduce the dependency errors for such a neuron in deeper layers. The symbolic expressions in practice are represented by coefficient vectors and constants.

\subsubsection{Symbolic propagation}
We denote the symbolic counterparts of the input vector $z^l$ and output vector $x^l$ of the neurons in layer $l$, as $ze^l$ and  $xe^l$ respectively. These expressions are symbolic linear expressions over the outputs of neurons from previous layers. To get the pre-activation bounds we calculate the lower and upper bounds of these expressions at each node. If we encounter a $\ast$-neuron, the linear relation breaks down and we can no longer propagate the expression from this node. Hence we introduce an independent variable for this node and hope to preserve the resultant expressions in deeper layers. $[\Lambda]^l$ is a diagonal interval matrix denoting the activation pattern of layer $l$, whose diagonal elements are intervals of the form $[1,1], [0,0]$ or $[0,1]$ corresponding to active, inactive or $\ast$- neurons (undecided) respectively. All other elements of this interval matrix are $[0,0]$. This interval matrix will be later used in calculating the Lipschitz upper bounds. Note that the number of $\ast$-neurons marked can only be exact or an overestimation. It is these $\ast$-neurons which we aim to remove successively by creating partitions.

\subsection{Sub-Problem Representation}
This section explains how a sub-problem, or equivalently a node in the Branch and Bound (BaB) tree, is represented. We denote a sub-problem as $\rho$ and a property $p$ of that sub-problem as $(p,\rho)$. Each sub-problem has a corresponding subset of the input space 
$(\psi,\rho)$ associated with it, defined by the set of half-space constraints $(H,\rho)$. Also, for every sub-problem there is an associated activation pattern of the network, $(A,\rho)$, where
\[A^l_i = \begin{cases}
1 & \text{if known}~ z^{(l)}_i> 0 ~\text{for all inputs in} ~\psi\\
0 & \text{if known}~ z^{(l)}_i< 0 ~\text{for all inputs in} ~\psi\\
\ast & \text{if known otherwise/ not known}
\end{cases}\]

Any sub-problem can be equivalently represented by a pair consisting of its set of half-space constraints and its activation pattern $\{H, A\}$. The half-space constraints of a sub-problem are of the form $zle^l_i<0$ or $zle^l_i>0$,
where $zle^l_i$ is a symbolic linear expression of only the input variables, corresponding to the $i^{th}$ neuron at layer $l$. $zle^0$ is the symbolic input vector.

\begin{algorithm}
\caption{Symbolic Propagation}
\label{symprop}
\begin{algorithmic}[1]
\Procedure{SymProp}{network $N$, input-domain $\mathcal{X}$}
    \State //$ze^l_i$ and $xe^l_i$ are the input and output expressions respectively For $ith$ neuron at level $l$
    \State Initialize $xe^0=ze^0=$ vector of input variables
    \For{l=1,...,L-1}
        \State $ze^l=W^{(l)}xe^{l-1}+b^l$
        \For {$i=1,...,n_l$}
            \If {$\underline{ze^l_i} > 0$}
                \State //Keep Dependency
                \State $xe^l_i=ze^l_i$
                \State $A^l_i=1, [\Lambda]^l_{ii}=[1,1]$
            \ElsIf {$ \overline{ze^l_i} < 0$}
                \State //Update to 0
                \State $xe^l_i=0$
                \State $A^l_i=0,[\Lambda]^l_{ii}=[0,0]$
            \Else
                \State //Introduce new variable $v^l_i$
                \State $\overline{v^l_i}= \overline{ze^l_i}$, $\underline{v^l_i}=0$
                \State $xe^l_i=v^l_i$
                \State $A^l_i=\ast, [\Lambda]^l_{ii}=[0,1]$
            \EndIf
        \EndFor
    \EndFor
    \State $ze^L=W^{(L)}xe^{L-1}+b^{(L)}$
    \State output-bounds=$[\underline{ze}^L\overline{ze}^L]$
\EndProcedure
\end{algorithmic}
\end{algorithm}


The initial sub-problem (root node of the BaB tree) is denoted as $\rho_I$, whose corresponding $\psi$ is $\mathcal{X}$ itself and activation pattern $A$ is as determined by Algorithm SymProp.
The constraints corresponding to the initial hyper-rectangular input space $\mathcal{X}$ are given as
$H=\bigcap\{zle^0_i< \overline{z}_i^{(0)}, i=1,\ldots,n^{(0)}\} \cap \bigcap\{zle^0_i> \underline{z}_i^{(0)}, i=1,\ldots,n^{(0)}\}$, which forms a bounded convex polytope. For any sub-problem, if the set of constraints $H$ form a bounded convex polytope and a hyper plane $zle^l_i=0$ cuts through this polytope, we get two partitions, which are also convex polytopes, given by the constraints,
$\{H\cap zle^l_i< 0\}$ and $\{H\cap zle^l_i > 0\}$. It follows from induction that any feasible set of constraints generated after any number of partitioning steps as stated above, forms a bounded convex polytope.

The use of open half-spaces, or strict inequalities, makes sure that when we have a sub-problem with no $\ast$-neurons, any feasible point in that region is actually a differentiable point. The reason is that the strict inequalities implies that the corresponding neurons takes non-zero values as inputs for all points in the feasible region, which in turn implies that the Jacobian is well defined. 

\subsection{Propagating linear relations}
In order to generate the half-space constraints corresponding to neurons at a layer $l$, we need to have the expressions $zle^l$ for that layer, which is possible only if there are no $\ast$-neurons present in previous layers. Therefore for any sub-problem, we simply propagate the linear relations $zle$ across the layers until we reach the last layer of the network or encounter a layer which contains a $\ast$-neuron (since by moving further we cannot preserve the linear relationship with the inputs anymore). $xle^l_i$ is the output expression of a neuron corresponding to the input $zle^l_i$. It is computed only if there are no $\ast$-neurons in a layer.

\[ zle^l=W^{(l)}xle^{l-1}+b^l\\\]
\[
    xle^l_i= \begin{cases}
    0 & \text{if}~ A^l_i=0\\
    zle^l_i & \text{if}~ A^l_i=1
    \end{cases}
\]
where $xle^0=zle^0$ is the symbolic input vector.
This process, called as LinProp (as used in subsequent algorithms), is similar to SymProp except that we don't need to evaluate any bounds or introduce any new variables.

\subsection{Lipschitz bounds}
In this section we describe the procedure to calculate valid upper bounds of the Lipschitz constant of a sub-problem (similar to \cite{fastlip}).  We use interval matrix multiplication to upper bound the Lipschitz constant for a sub-problem. Similar to the Jacobian $J_f(x)$ for a single point $x$ as described before, we can represent the notion of a Jacobian matrix for a set of points $X$ by an interval matrix where each element is an interval which bounds the partial derivatives of all the points. We have 
\[[J(X)] = W^{(L)} [\Lambda]^{(L-1)}W^{(L-1)}\ldots [\Lambda]^{(1)}W^{(1)}\]
where $[\Lambda]^l$ is an interval matrix used to denote the activation pattern for layer $l$, as described in SymProp. The intervals $[0,1]$ used to represent $\ast$-neurons, takes into account both possible activation states to calculate the extreme cases of lower and upper bounds of the partial derivatives.
Once we obtain this $[J]$ matrix, calculated using interval matrix multiplication, we construct an ordinary matrix $U$ of the same size as that of $[J]$ where each element upper bounds the absolute values of the corresponding intervals in $[J]$. It is to be noted that the interval bounds are over-approximation of the actual values.  The $p$-norm of the constructed $U$ matrix gives us an upper-bound of the local Lipschitz constant for the corresponding sub-problem. In case we have no $\ast$-neurons (corresponds to a piece-wise linear region), we simply return the $p$-norm of the Jacobian at that region.


\begin{algorithm}
\caption{Calculating Lipschitz Bounds}
\begin{algorithmic}[1]
\Procedure{LipschitzBounds}{sub-problem $\rho$}
    \If{$\rho$ has no $\ast$ neurons}
        \State //$\psi$ is a linear region, hence Jacobian is defined
        \State\Return {$\|J_f\|_p$}
    \Else
        \State Initialize $[J]_{ij}=[W_{ij}^{(1)}, W_{ij}^{(1)}], \forall i,j$
        \For{$l=2, \ldots,L$}
            \State //interval matrix multiplication
            \State $[J]=W^{(l)}[\Lambda]^{(l-1)}[J]$
        \EndFor
        \State Define $\{ U| U_{ij}=max(abs(\underline{[J]_{ij}}), abs(\overline{[J]_{ij}})), \forall i,j\}$
        \State \Return $\|U\|_p$
    \EndIf
\EndProcedure
\end{algorithmic}
\end{algorithm}


\subsection{Branching}
\label{sec:split}
\label{sec:split}
The main idea behind the branching step is to create a partition in the polytope associated with a sub-problem and compute the upper-bound of the local Lipschitz constant of each partition to get tighter estimates. The partitions are made in such a way such that inputs from one part activates a specific neuron while for the inputs from the other part, the neuron is inactive.
A new set of constraints is created by adding a new constraint ($zle^l_i<0$ in one case, $zle^l_i>0$ in the other) corresponding to a $\ast$-neuron at layer $l$, to the existing set of half-space constraints $H$. A related idea of solving linear programs to get refined values after adding constraints similar to this is discussed in \cite{neurify}. 
If the constraints set is feasible, we create a sub-problem with the new set of constraints and activation pattern (based on the new half-space constraint added). The next steps include propagating the linear expressions, $zle$, and also removing some of the  $\ast$-neurons whose states can be determined to be active/inactive with respect to the new constraint set.
Finally we calculate the Lipschitz bound for the newly created sub-problem and add it to the set of sub-problems. Note that the Lipschitz bounds of a branched problem can only decrease, since we reduce uncertainties.

\begin{algorithm}
\caption{Branching}
\begin{algorithmic}[1]
\Procedure{Branch}{Sub-problem $\rho$}
\State $t\leftarrow$ first layer with $\ast$-neurons
\State select a $\ast$-neuron($i$th neuron at layer $t$)
\State $H_0, H_1 \leftarrow \{(H,\rho)\cap zle^t_i< 0\}, \{(H,\rho)\cap zle^t_i> 0\}$
\State $S\leftarrow S\backslash \{\rho\}$
\For{$r\in\{0,1\}$ }
    \If {$Feasible(H_r)$}
        \State Create sub-problem $\rho_r \equiv \{H_r, (A,\rho)\}$
        \State  $(A^t_i, \rho_r)\leftarrow r, ([\Lambda]^t_{ii}, \rho_r)\leftarrow [r,r]$
        \State LinProp($\rho_r$)//propagating linear relations
        \State FFilter($\rho_r$)//feasibility filter
        \State $(L_{ub},\rho_r)=LipschitzBounds(\rho_r)$
        \State $S\leftarrow S\cup \rho_r$
        \If{$\rho_r$ has no $\ast$-neurons}
            \State //update lower bound
            \State $glb \leftarrow max(glb, (L_{ub},\rho_r))$
        \EndIf
    \EndIf
\EndFor
\EndProcedure
\end{algorithmic}
\end{algorithm}

\subsection{Feasibility Filter}
When a new sub-problem is created by creating a partition in a polytope, the smaller partitions may be able to decide the states (active/inactive) of more than one neuron which were marked as $\ast$-neurons before the partitioning. It is easy to see that identifying such a neuron early on in the BaB tree is better as it prevents repeating the process (feasibility checks) for several (in worst case, exponential) sub-problems which are generated later. 
We use a simple heuristic to do this. 
We keep fixing the states of $\ast$-neurons which are decidable by the new constraints till we encounter a neuron which still maintains both active and inactive states depending on the inputs. We check the feasibility of both $(H\cap zle^l_i<0)$ and $(H\cap zle^l_i>0)$ for a $\ast$-neuron at layer $l$. Based on which of these two is feasible we decide the activation state of the neuron. If both of them are feasible then this is a $\ast$-neuron with respect to $H$, and we terminate this step. Later, if we need to branch on this sub-problem we shall choose to branch on the $\ast$-neuron which was already found to maintain both active and inactive states. This strategy will always create two feasible sub-problems. We use this process, called FFilter, in combination with LinProp (for generating half-space constraints) to reduce $\ast$-neurons across layers.


\section{Final Algorithm}
\label{sec:final-algorithm}

This section puts together all the individual components discussed before and provides the main branch and bound framework, as provided in Algorithm~\ref{alg:final}. Given the network parameters and the input region of interest, the first step is to run symbolic propagation to mark the state of each neurons as active, inactive or $\ast$. This gives us the initial activation pattern. The first sub-problem is created with this activation pattern and half-space constraints given by the bounding constraints of the input region. The corresponding Lipschitz upper bound is also calculated. 
The set of sub-problems is initialized with this sub-problem. We use a max heap data structure to represent the sub-problems, sorted by the Lipschitz upper-bound of the sub-problems. $glb$ and $gub$ are used to keep track of the lower and upper bounds of $\mathcal{L}_p(f,\mathcal{X})$ respectively. The algorithm iteratively selects a sub-problem from the set with the highest Lipschitz upper bound and branches on it. Each newly created sub-problem with its own set of half-space constraints, activation pattern and corresponding Lipschitz upper bound is pushed into the heap. Also, if a sub-problem has no $\ast$-neurons it is a valid lower-bound for $\mathcal{L}_p(f,\mathcal{X})$, and we use it to compare and update $glb$.

\begin{algorithm}
\caption{Final Algorithm}
\label{alg:final}
\begin{algorithmic}[1]
\Procedure{LipBaB}{network $N$, input domain $\mathcal{X}$, approximation factor $k$}
\State Initial constraints $H \leftarrow$ Bounding constraints of $\mathcal{X}$
\State Initial activation $A \leftarrow$ SymProp($N, \mathcal{X}$)
\State Create initial sub-problem $\rho_I \equiv \{H,A\}$
\State LinProp($\rho_I$) //propagating linear relations
\State  $(L_{ub},\rho_I)\leftarrow LipschitzBounds(\rho_I)$
\State $S\leftarrow S\cup \rho_I$
\State $gub,glb\leftarrow (L_{ub},\rho_I),0$ // initialize lower and upper bounds of $\mathcal{L}_p(f,\mathcal{X})$
\If {$\rho_I$ has no $\ast$-neurons}
    \State $glb \leftarrow (L_{ub},\rho_I)$
\EndIf
\While{\textit{True}}
    \State $\rho' \leftarrow \arg\max_{\rho \in S}(L_{ub},\rho)$
    \State $gub\leftarrow (L_{ub},\rho')$
    \If{$gub\leq k.glb$} 
        \State //k=1 implies $gub=glb=\mathcal{L}_p$
        \State break 
    \Else
        \State $Branch(\rho')$
    \EndIf
\EndWhile
\State\Return{ $gub$}
\EndProcedure
\end{algorithmic}
\end{algorithm}

For an approximation factor $k$, we can terminate when $gub \leq k.glb$ since we know that $\mathcal{L}_p(f,\mathcal{X})$ lies between $glb$ and $gub$. While calculating the exact local Lipschitz constant the algorithm is terminated when $gub=glb$ or equivalently, the sub-problem at the top of the heap as no $\ast$-neurons. This means that the input region corresponding to the sub-problem is actually a piece-wise linear region and therefore the local Lipschitz upper-bound of that region is exact instead of an upper bound. Also, since this is the highest among all the other sub-problems, by definition, this is the exact local Lipschitz constant of the entire input region $\mathcal{X}$. The algorithm returns tighter estimates of the Lipschitz upper-bound iteratively until convergence, hence terminating early because of constraints like time/memory, will always provide us with a valid upper-bound.

To compute the global Lipschitz constant we need to simply mark every neuron as $\ast$-neuron. This takes care of all the possible activation patterns throughout $\mathbb{R}^n$ and there is no need for any initial input constraints. The rest of the procedure is same. Note, in this case the feasible region of sub-problems are not necessarily bounded. 


\subsection{Analysis}
To prove both the correctness of the Lipschitz upper-bound calculation and convergence of the BaB algorithm we use some simple properties of interval arithmetic.

\subsubsection{Correctness of Lipschitz upper bounds}
The reason for using an interval $[0,1]$ for a $\ast$-neuron is to consider both the activation states of the neuron to calculate the extreme cases of lower and upper bound. Since any basic arithmetic operation on intervals bounds the resultant intervals considering both the end points of every interval (which in our case are the activation states $0$ and $1$), we know the resultant interval bounds on the partial derivatives obtained after interval matrix multiplication are exact/over-approximated bounds which has already covered all the cases of possible activation patterns.

\subsubsection{Convergence of the algorithm}
First, we note that if the activation state of any neuron is decided in a node of the BaB  tree, then it stays decided for any node branched from it. This is simply because if a neuron is active/inactive for a set of inputs, then it is also active/inactive accordingly for any of its subset. 
Proving the convergence of the Lipschitz bounds depends on the property that if any interval which is used in some basic arithmetic operations is replaced by a sub-interval then the the resulting interval bounds can only tighten. Every time we branch we reduce the number of $\ast$-neurons by at least one, in each branched problem (change interval $[0,1]$ to sub-intervals $[0,0]$/$[1,1]$), which implies that the Lipschitz bounds of the branched sub-problems can only decrease. This along with the fact that the number of $\ast$-neurons are finite and strictly reduce along any path from the root of the BaB tree, proves the convergence of the algorithm with termination.

\subsubsection{Number of sub-problems}
Since we have a method which guarantees that every branching will create two feasible sub-problems, i.e., the BaB tree will be a full binary tree, we can show that the number of sub-problems generated can be bounded by at most $2p-1$, where $p$ is the number of piece-wise linear regions in the input space $\mathcal{X}$. We know that a sub-problem corresponding to a piece-wise linear region has no $\ast$-neurons and vice versa, and is therefore, a terminal sub-problem. Hence, if we keep generating feasible sub-problems till we have no more to split, we will be having at most $p$ leaf nodes (terminal sub-problems), each of which corresponds to a piece-wise linear region. Therefore the total number of sub-problems generated will be $2p-1$, since the BaB tree is a full binary tree. However, in practice, the total number of sub-problems generated is usually much less, because of branch and bound.




\section{Implementation and Experiments}
\label{sec:experiments}

In this section we provide experimental results to illustrate the working of our algorithm on various parameters. All the experiments were done on Google Colab notebooks. The implementation is done using Python (available in \href{https://github.com/pyrobits/LipBaB}{GitHub}). For feasibility checking, we used the `GLPK' Linear Programming solver with a default tolerance limit of 1e-7. The MLPClassifier module from scikit-learn was used to train the networks. The local Lipschitz computation was done on the networks considering the input region $[0,1]^4$ for the Iris data-set and $[0,0.1]^{10}$ for the synthetic data-sets. The synthetic data-set  consisted of 2000 data points of 10 dimensions and 3 classes, generated using scikit-learn. Note that the choice of input region is arbitrary.

\begin{table}[]
\begin{tabular}{|c|c|c|c|c|c|c|c|c|c|}
\hline
\multirow{2}{*}{Network}  & \multirow{2}{*}{p-norm}  & \multicolumn{2}{c|}{First estimation}                                 & \multicolumn{2}{c|}{2-approximation}                                                        & \multicolumn{2}{c|}{1.5-approximation}                                                                                              & \multicolumn{2}{c|}{Exact}    \\ \cline{3-10} 
                                                                       &         & Time   & Value                                                               & Time                                                           & Value                                                            & Time                                                             & Value                                                            & Time                                                             & Value                                                            \\ \hline
                                                                       
\begin{tabular}[c]{@{}c@{}}Iris\_Network\\ (4,5,5,3)\end{tabular} & \begin{tabular}[c]{@{}c@{}}1\\ 2\\ inf\end{tabular} & \begin{tabular}[c]{@{}c@{}}0.02s\\ 0.02s\\ 0.02s\end{tabular} & \begin{tabular}[c]{@{}c@{}}8.776\\ 8.810\\ 14.663\end{tabular} & \begin{tabular}[c]{@{}c@{}}0.06s\\ 0.07s\\ 0.06s\end{tabular} & \begin{tabular}[c]{@{}c@{}}6.874\\ 7.098\\ 12.606\end{tabular} & \begin{tabular}[c]{@{}c@{}}0.06s\\ 0.06s\\ 0.06s\end{tabular} & \begin{tabular}[c]{@{}c@{}}6.874\\ 7.098\\ 12.606\end{tabular} & \begin{tabular}[c]{@{}c@{}}0.11s\\ 0.09s\\ 0.06s\end{tabular} & \begin{tabular}[c]{@{}c@{}}5.959\\ 6.772\\ 12.606\end{tabular} \\ \hline                                                                                          
\begin{tabular}[c]{@{}c@{}}SD\_Network1\\ (10,15,10,3)\end{tabular}    & \begin{tabular}[c]{@{}c@{}}1\\ 2\\ inf\end{tabular} & \begin{tabular}[c]{@{}c@{}}0.04s\\ 0.05s\\ 0.04s\end{tabular} & \begin{tabular}[c]{@{}c@{}}15.105\\ 13.019\\ 25.243\end{tabular}    & \begin{tabular}[c]{@{}c@{}}0.17s\\ 0.16s\\ 0.18s\end{tabular}  & \begin{tabular}[c]{@{}c@{}}12.704\\ 10.658\\ 19.543\end{tabular} & \begin{tabular}[c]{@{}c@{}}0.17s\\ 0.16s\\ 0.18s\end{tabular}    & \begin{tabular}[c]{@{}c@{}}12.704\\ 10.658\\ 19.543\end{tabular} & \begin{tabular}[c]{@{}c@{}}0.25s\\ 0.28s\\ 0.27s\end{tabular}    & \begin{tabular}[c]{@{}c@{}}10.413\\ 9.531\\ 16.275\end{tabular}  \\ \hline
\begin{tabular}[c]{@{}c@{}}SD\_Network2\\ (10,20,15,10,3)\end{tabular} & \begin{tabular}[c]{@{}c@{}}1\\ 2\\ inf\end{tabular} & \begin{tabular}[c]{@{}c@{}}0.06s\\ 0.07s\\ 0.06s\end{tabular} & \begin{tabular}[c]{@{}c@{}}101.705\\ 101.940\\ 182.988\end{tabular} & \begin{tabular}[c]{@{}c@{}}1.01s\\ 1.79s\\ 2.26s\end{tabular}  & \begin{tabular}[c]{@{}c@{}}70.928\\ 55.969\\ 82.938\end{tabular} & \begin{tabular}[c]{@{}c@{}}1.01s\\ 1.79s\\ 2.26s\end{tabular}    & \begin{tabular}[c]{@{}c@{}}70.928\\ 55.969\\ 82.938\end{tabular} & \begin{tabular}[c]{@{}c@{}}3.08s\\ 3.35s\\ 2.97s\end{tabular}    & \begin{tabular}[c]{@{}c@{}}48.049\\ 40.057\\ 72.286\end{tabular} \\ \hline
\begin{tabular}[c]{@{}c@{}}SD\_Network3\\ (10,30,30,30,3)\end{tabular} & \begin{tabular}[c]{@{}c@{}}1\\ 2\\ inf\end{tabular} & \begin{tabular}[c]{@{}c@{}}0.20s\\ 0.20s\\ 0.20s\end{tabular} & \begin{tabular}[c]{@{}c@{}}131.727\\ 139.808\\ 272.416\end{tabular} & \begin{tabular}[c]{@{}c@{}}9.43s\\ 19.11s\\ 18.90s\end{tabular} & \begin{tabular}[c]{@{}c@{}}33.890\\ 30.684\\ 63.035\end{tabular} & \begin{tabular}[c]{@{}c@{}}22.24s\\ 25.94s\\ 23.88s\end{tabular} & \begin{tabular}[c]{@{}c@{}}25.871\\ 28.474\\ 55.311\end{tabular} & \begin{tabular}[c]{@{}c@{}}56.00s\\ 78.88s\\ 59.98s\end{tabular} & \begin{tabular}[c]{@{}c@{}}19.370\\ 19.463\\ 39.111\end{tabular} \\ \hline
\end{tabular}
\caption{\label{tab:experiments} Lipschitz computation for different approximation factors}
\end{table}

\begin{figure}[h]
\centering
\begin{tabular}{ll}
\includegraphics[scale=0.4]{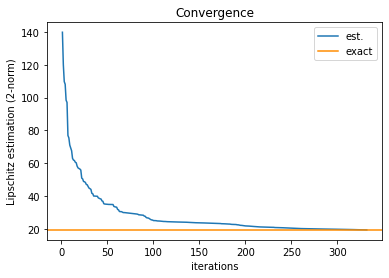}
&
\includegraphics[scale=0.4]{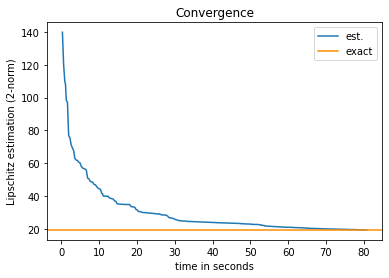}
\end{tabular}
\caption{Convergence of the algorithm for a network with layer sizes (10,30,30,30,3)}
\label{fig:convergence}
\end{figure}

It was observed that the algorithm achieves a good approximation factor (sufficient for most practical cases) within a reasonable time, but gradually converges slowly as the number of sub-problems increases exponentially. Also, the output bounds calculated for a network using SymProp was found to be much tighter than that calculated using naive interval propagation. It was found as expected that the optimization strategy,  FFilter, improved the performance significantly.

Achieving arbitrary precision with regards to feasibility tolerance for constraints is not possible for solvers with finite tolerance limits. For this reason, in very rare cases, it may falsely validate some set of constraints to be feasible if the precision demand is more than the tolerance limit of solvers. This might validate false activation patterns which can (not necessarily) cause the algorithm to report a value larger than the true exact Lipschitz constant.

\section{Conclusion}
We provide techniques to calculate exact bounds of the Lipschitz constant of neural networks, which has several important applications in deep learning. Our main contribution is that this technique can be applied for exact/approximate computation of the local/global Lipschitz constant for any $p$-norm. 

This work discusses several ideas within an unified branch and bound framework which can be extended to related problem areas of neural network verification like output bounds computation, robustness estimation etc.. The algorithm also provides local information about the sensitivity of the neural network corresponding to different parts of the input space. The branch and bound strategy discussed here, based on the partitioning of input space, can be potentially extended to iteratively refine the estimation of other important properties of neural networks which varies across the input space. The exact computation of the Lipschitz constant for different norms may prove to be useful for other theoretical studies of neural networks as well as real-life safety critical scenarios requiring formal guarantees about the behaviour of neural networks.

\bibliographystyle{plain}
\bibliography{references}

\begin{thebibliography}{10}

\bibitem{invres}
Jens Behrmann, Will Grathwohl, Ricky T.~Q. Chen, David Duvenaud, and
  Joern-Henrik Jacobsen.
\newblock Invertible residual networks.
\newblock In Kamalika Chaudhuri and Ruslan Salakhutdinov, editors, {\em
  Proceedings of the 36th International Conference on Machine Learning},
  volume~97 of {\em Proceedings of Machine Learning Research}, pages 573--582.
  PMLR, 09--15 Jun 2019.

\bibitem{flowres}
Ricky T.~Q. Chen, Jens Behrmann, David Duvenaud, and Jörn-Henrik Jacobsen.
\newblock Residual flows for invertible generative modeling, 2020.

\bibitem{lipsdp}
Mahyar Fazlyab, Alexander Robey, Hamed Hassani, Manfred Morari, and George~J.
  Pappas.
\newblock Efficient and accurate estimation of {L}ipschitz constants for deep
  neural networks.
\newblock {\em CoRR}, abs/1906.04893, 2019.

\bibitem{lipmilp}
Matt Jordan and Alexandros~G. Dimakis.
\newblock Exactly computing the local {L}ipschitz constant of {ReLU} networks,
  2020.

\bibitem{deepmind}
Hyunjik Kim, George Papamakarios, and Andriy Mnih.
\newblock The {L}ipschitz constant of self-attention, 2020.

\bibitem{lipopt}
Fabian Latorre, Paul Rolland, and Volkan Cevher.
\newblock Lipschitz constant estimation of neural networks via sparse
  polynomial optimization, 2020.

\bibitem{spectralNorm}
Takeru Miyato, Toshiki Kataoka, Masanori Koyama, and Yuichi Yoshida.
\newblock Spectral normalization for generative adversarial networks, 2018.

\bibitem{estWassDist}
Gabriel Peyré and Marco Cuturi.
\newblock Computational optimal transport: With applications to data science.
\newblock {\em Foundations and Trends® in Machine Learning}, (5-6):355--607.

\bibitem{generalizationBounds}
Jure Sokoli{\'c}, Raja Giryes, Guillermo Sapiro, and Miguel~RD Rodrigues.
\newblock Robust large margin deep neural networks.
\newblock {\em IEEE Transactions on Signal Processing}, 65(16):4265--4280,
  2017.

\bibitem{lipschitzmargin}
Yusuke Tsuzuku, Issei Sato, and Masashi Sugiyama.
\newblock Lipschitz-margin training: Scalable certification of perturbation
  invariance for deep neural networks, 2018.

\bibitem{NIPS2018-seqlip}
Aladin Virmaux and Kevin Scaman.
\newblock Lipschitz regularity of deep neural networks: analysis and efficient
  estimation.
\newblock In S.~Bengio, H.~Wallach, H.~Larochelle, K.~Grauman, N.~Cesa-Bianchi,
  and R.~Garnett, editors, {\em Advances in Neural Information Processing
  Systems 31}, pages 3835--3844. Curran Associates, Inc., 2018.

\bibitem{neurify}
Shiqi Wang, Kexin Pei, Justin Whitehouse, Junfeng Yang, and Suman Jana.
\newblock Efficient formal safety analysis of neural networks.
\newblock {\em CoRR}, abs/1809.08098, 2018.

\bibitem{reluval}
Shiqi Wang, Kexin Pei, Justin Whitehouse, Junfeng Yang, and Suman Jana.
\newblock Formal security analysis of neural networks using symbolic intervals.
\newblock In {\em Proceedings of the 27th USENIX Conference on Security
  Symposium}, SEC'18, page 1599–1614, USA, 2018. USENIX Association.

\bibitem{fastlip}
Tsui-Wei Weng, Huan Zhang, Hongge Chen, Zhao Song, Cho-Jui Hsieh, Duane Boning,
  Inderjit~S. Dhillon, and Luca Daniel.
\newblock Towards fast computation of certified robustness for {ReLU} networks.
\newblock In {\em International Conference on Machine Learning (ICML)}, jul
  2018.

\bibitem{clever}
Tsui-Wei Weng, Huan Zhang, Pin-Yu Chen, Jinfeng Yi, Dong Su, Yupeng Gao,
  Cho-Jui Hsieh, and Luca Daniel.
\newblock Evaluating the robustness of neural networks: An extreme value theory
  approach, 2018.

\end{thebibliography}

\end{document}